\documentclass[english,a4paper,12pt]{article}

\usepackage{amsxtra, amsfonts, amssymb, amstext, amsmath, mathtools}
\usepackage{amsthm}
\usepackage{booktabs}
\usepackage{nicefrac}
\usepackage{xspace}
\usepackage{url}
\usepackage{graphics,color}
\usepackage[algo2e,ruled,vlined,linesnumbered]{algorithm2e}
\usepackage{wrapfig}
\usepackage{lmodern}
\usepackage{array}
\newcolumntype{?}{!{\vrule width 1.5pt}}

\newtheorem{theorem}{Theorem}
\newtheorem{lemma}[theorem]{Lemma}
\newtheorem{corollary}[theorem]{Corollary}

\newcommand{\mpoea}{\mbox{${(\mu+1)}$~EA}\xspace}
\newcommand{\mpoga}{\mbox{${(\mu+1)}$~GA}\xspace}

\newcommand{\NSGA}{\mbox{NSGA}\nobreakdash-II\xspace}

\newcommand{\jump}{\textsc{Jump}\xspace}
\newcommand{\oneminmax}{\textsc{OneMinMax}\xspace}
\newcommand{\cocz}{\textsc{COCZ}\xspace}
\newcommand{\lotz}{\textsc{LOTZ}\xspace}
\newcommand{\ojzj}{\textsc{OneJumpZeroJump}\xspace}
\newcommand{\onejumpzerojump}{\textsc{OneJumpZeroJump$_{n,k}$}\xspace}
\newcommand{\jumpnk}{\textsc{Jump$_{n,k}$}\xspace}

\DeclareMathOperator{\cDis}{cDis}

\let\originalleft\left
\let\originalright\right
\renewcommand{\left}{\mathopen{}\mathclose\bgroup\originalleft}
\renewcommand{\right}{\aftergroup\egroup\originalright}

\usepackage{hyperref}
\begin{document}

\title{Runtime Analysis for the NSGA-II:\\ Provable Speed-Ups From Crossover}

\author{Benjamin Doerr\thanks{Laboratoire d'Informatique (LIX), CNRS, \'Ecole Polytechnique, Institut Polytechnique de Paris, Palaiseau, France} \and Zhongdi Qu\thanks{Laboratoire d'Informatique (LIX), CNRS, \'Ecole Polytechnique, Institut Polytechnique de Paris, Palaiseau, France}}

\maketitle

\begin{abstract}
  Very recently, the first mathematical runtime analyses for the NSGA-II, the most common multi-objective evolutionary algorithm, have been conducted. Continuing this research direction, we prove that the NSGA-II optimizes the OneJumpZeroJump benchmark asymptotically faster when crossover is employed. Together with a parallel independent work by Dang, Opris, Salehi, and Sudholt, this is the first time such an advantage of crossover is proven for the NSGA-II. Our arguments can be transferred to single-objective optimization. They then prove that crossover can speed up the $(\mu+1)$ genetic algorithm in a different way and more pronounced than known before. Our experiments confirm the added value of crossover and show that the observed advantages are even larger than what our proofs can guarantee.   
\end{abstract}

{\sloppy 
\section{Introduction}

The theory of randomized search heuristics~\cite{AugerD11} has greatly improved our understanding of heuristic search, in particular, via mathematical runtime analyses. Due to the complicated nature of the stochastic processes describing runs of these algorithms, mostly very simple, often synthetic, heuristics could be analyzed so far. Very recently, however, a runtime analysis was conducted~\cite{ZhengLD22} for the \NSGA, the most common evolutionary multi-objective (EMO) algorithm~\cite{DebPAM02} (47000 citations on Google scholar). This work was quickly followed up in different directions~\cite{ZhengD22gecco, BianQ22, DoerrQ22ppsn, ZhengD22arxivmany, DoerrS22arxiv}. The majority of these studies regards a simplified version of the \NSGA that does not use crossover. Only in~\cite{BianQ22} an \NSGA with crossover is analyzed, but the runtime guarantees shown there were not better than those shown before for the mutation-only \NSGA.

In this work, we conduct the (together with the parallel work~\cite{DangOSS23}) first  mathematical runtime analysis of the \NSGA with crossover that proves runtime guarantees asymptotically stronger than those known for the \NSGA without crossover. To this aim, we regard the \ojzj benchmark~\cite{DoerrZ21aaai}, which is a bi-objective version of the classic \jump benchmark intensively studied in the analysis of single-objective search heuristics. The mutation-based \NSGA with population size~$N$ at least four times the size of the Pareto front computes the Pareto front of this problem in expected time (number of function evaluations)  $O(Nn^k)$, where $n$ is the length of the bit-string representation and $k$ is the jump size, a difficulty parameter of the problem~\cite{DoerrQ22ppsn}. The authors say that they believe this bound to be tight, but prove a lower bound of $\Omega(n^k)$ only. That work includes preliminary experiments that show performance gains from crossover, but without proving them or giving additional insights on why they arise. 

Our main result is that the \NSGA using crossover with a constant rate solves this problem in expected time $O(N^2 n^k / \Theta(k)^k)$. This is faster than the previous bound and the $\Omega(n^k)$ lower bound when $k \ge c \log(N) / \log\log(N)$, where $c$ is a suitable constant. The key to proving this performance advantage is a careful analysis of how the \NSGA arrives at having substantially different solutions with the same objective value in its population. This appears to be an advantage of the \NSGA over algorithms previously analyzed mathematically such as the SEMO or global SEMO, where for each objective value at most one solution is kept in the population. 

As a side result, we observe that our arguments for proving diversity in the population can also be used in single-objective optimization. This way we show that the $(\mu+1)$ genetic algorithm (GA) -- without additional diversity mechanisms or other adjustments favoring crossover -- optimizes the jump function with encoding length~$n$ and jump size~$k$ in time $O(\mu^{k+1} n^k / \Theta(k)^k)$. This result holds for all $n$, $2 \le k = o(\sqrt n)$, and $2 \le \mu$. For many parameter values, this compares favorably with the known $\Omega(n^k)$ runtime of the mutation-only analog of the \mpoga, the so-called \mpoea. Different from the previous analyses of the \mpoga on \jump in~\cite{DangFKKLOSS18,DoerrEJK23arxiv}, our result shows super-constant speed-ups already for constant population sizes and our speed-ups increase with the problem difficulty~$k$. 

Our experimental results confirm the advantages of crossover for both the \NSGA and the \mpoga. The advantages observed are even stronger than what our proofs guarantee. We note that it is not uncommon that mathematical runtime analyses cannot prove all advantages of an algorithm visible in experiments. In return, they give proven results and  explanations on how the advantages arise. 

Overall, our results give a strong mathematical justification for using the \NSGA with crossover as proposed in the original work~\cite{DebPAM02}. Our new analysis methods promise to be applicable also to other randomized search heuristics using crossover, as demonstrated brief{}ly on the \mpoga. 

We note that a similar result appeared in the parallel independent work~\cite{DangOSS23}. There a problem was constructed for which the mutation-based \NSGA has an exponential runtime, but with crossover the runtime becomes polynomial. From the main proof arguments, it appears unlikely that this result has a single-objective analog. 

\section{Previous Works}\label{sec:previous}

Since this work conducts a mathematical runtime analysis of a classic, crossover-based EMO algorithm, let us brief{}ly describe the state of the art in runtime analysis with respect to EMO algorithms and crossover. 

Mathematical runtime analyses have a long history in the field of heuristic search, cf.~\cite{SasakiH88,Back93,Gutjahr08,NeumannW15,LissovoiO18,LissovoiOW19} or the textbooks~\cite{NeumannW10,AugerD11,Jansen13,DoerrN20}. While often restricted to simple algorithms and problems, the mathematical nature of these works has led to many deep insights that could not have been obtained with empirical methods. 

The mathematical runtime analysis of EMO algorithms was started in~\cite{LaumannsTZWD02,Giel03,Thierens03}. As in the theory of single-objective heuristics, these and most subsequent works analyzed simple synthetic algorithms. At the last AAAI conference, the first runtime analysis of the \NSGA algorithm~\cite{DebPAM02}, the most common EMO algorithm in practice, was presented~\cite{ZhengLD22}. It proved that the \NSGA with suitable population size can find the Pareto front of the \oneminmax and \lotz benchmarks in the same asymptotic runtimes that were shown previously for the SEMO and GSEMO synthetic algorithms. This result is remarkable in that the \NSGA has much more complex population dynamics than the (G)SEMO. In particular, it can lose desirable solutions due to the complex selection mechanism building on non-dominated sorting and crowding distance (and this was proven to happen consistently when the population size is equal to the size of the Pareto front~\cite{ZhengLD22}).

In~\cite{ZhengD22gecco}, it was proven that the \NSGA with a smaller population size can still compute good approximations of the Pareto front. For this, however, a mild modification of the selection mechanism, proposed earlier~\cite{KukkonenD06}, was needed. In~\cite{DoerrQ22ppsn}, the first mathematical runtime analysis of the \NSGA on a problem with multimodal objectives was conducted. Again, the \NSGA was found to be as effective as the (G)SEMO algorithms when the population size was chosen suitably. 

The three works just discussed consider the \NSGA as proposed in the original work~\cite{DebPAM02} except that they do not employ crossover. The only runtime analysis of the \NSGA with crossover (prior to this work and the parallel work~\cite{DangOSS23}) was conducted in~\cite{BianQ22}, namely on the \cocz, \oneminmax, and \lotz benchmarks. The runtime guarantees shown there agree with the ones in~\cite{ZhengLD22}, so no advantage of crossover was shown. In~\cite{BianQ22}, also a novel selection mechanism that gives considerable speed-ups was proposed.

The question of whether crossover, that is, generating new solutions from two existing ones, is useful or not, is as old as the area of genetic algorithms. Despite its importance, we are still far from having a satisfying answer. 
It is clear that at all times, the vast majority of evolutionary algorithms used in practice employ crossover. A solid scientific proof for the usefulness of crossover, however, is still missing. 

As an example of one of many unsuccessful attempts to explain the power of crossover, we cite the building-block hypothesis (BBH)~\cite{Holland75}, which states that crossover is effective because it allows combining small profitable segments of different solutions. While very convincing on the intuitive level, a simple experimental analysis on a synthetic problem perfectly fulfilling the assumptions of the BBH raised some doubts. In~\cite{mitchell92royal}, a simple randomized hill-climber was found to solve the proposed Royal-Road problem around ten times faster than a comparably simple genetic algorithm using crossover.

Theoretical approaches have been used as well to demonstrate the usefulness of crossover, but again mostly with limited success. The works closest to ours in single-objective optimization regard the \jump benchmark, which looks like an ideal example for exploiting crossover. Surprisingly, it was much harder than expected to show that crossover is profitable here. The early analyses~\cite{JansenW02, KotzingST11} managed to show an advantage from crossover only with an unrealistically small crossover probability. In~\cite{DangFKKLOSS18}, for the first time it was shown that a standard \mpoga with standard parameters optimizes jump functions faster with crossover than without. Other examples proving advantages of crossover include~\cite{Sudholt05, LehreY11,DoerrHK12,DoerrDE15,DangFKKLOSS16,Sutton21,AntipovDK22}, but they appear specific to a particular problem or a particular algorithm.

In multi-objective optimization, so far only three mathematical runtime analyses showing an advantage of crossover exist. 
In~\cite{NeumannT10}, it was proven that crossover is useful when solving multi-criteria versions of the all-pairs-shortest-path (APSP) problem. This work follows the ideas of the corresponding single-objective result~\cite{DoerrHK12}. Both are somewhat problem-specific in the sense that the formulation of the APSP problem automatically leads to a very strong diversity mechanism, from which crossover substantially profits in both works. 

In~\cite{QianYZ13}, a substantial advantage from crossover, among others, for the classic \cocz problem was shown. For this, however, a novel initialization is used which, for this problem, results in the initial and all subsequent populations containing both the all-ones and the all-zeroes string. It is clear that this strong diversity greatly helps crossover to become effective. 

Both these works regard the SEMO and GSEMO algorithms, two synthetic algorithms proposed for theoretical analyses~\cite{LaumannsTZWD02,Giel03}. Given this and the particularities of the results, the diversity mechanism implicit in the APSP problem and the particular initialization in~\cite{QianYZ13}, it is not clear to what extent the insight that crossover is beneficial can be expected to generalize to the broader EMO field. 

The third work showing an advantage of crossover in EMO optimization is~\cite{HuangZCH19}. Since it discusses a decomposition-based algorithm, it is very far from our work and we do not detail it further.

\section{Preliminaries}\label{sec:prelim}
\subsection{The NSGA-II Algorithm}\label{sec:prelim-alg}
In the interest of brevity, we only give a brief overview of the algorithm here and refer to~\cite{DebPAM02} for a more detailed description. 
The NSGA-II uses two metrics to completely order any population, which are rank and crowding distance. The ranks are defined recursively based on the dominance relation. All non-dominated individuals have rank~1. Then, given that the individuals of ranks $1, \dots, k$ are defined, the individuals of rank $k+1$ are those not dominated except by individuals of rank $k$ or smaller. This defines a partition of the population into sets $F_1$, $F_2$,\dots such that $F_i$ contains all individuals with rank $i$. Individuals with lower ranks are preferred. The crowding distance, denoted by $\cDis(x)$ for an individual~$x$, is used to compare individuals of the same rank. To compute the crowding distances of individuals of rank $i$ with respect to a given objective function $f_j$, we first sort the individuals in ascending order according to their $f_j$ objective values. The first and last individuals in the sorted list have infinite crowding distance. For the other individuals, their crowding distance is the difference between the objective values of their left and right neighbors in the sorted list, normalized by the difference between the minimum and maximum values. The final crowding distance of an individual is the sum of its crowding distances with respect to each objective function. Among individuals of the same rank, the ones with higher crowding distances are preferred.

The algorithm starts with a random initialization of a parent population of size $N$. At each iteration, $N$ children are generated from the parent population via a variation operator, and $N$ best individuals among the combined parent and children population survive to the next generation based on their ranks and, as a tie-breaker, the crowding distance (remaining ties are broken randomly). In each iteration, the critical rank $i^*$ is the rank such that if we take all individuals of ranks smaller than $i^*$, the total number of individuals will be less than or equal to $N$, but if we also take all individuals of rank $i^*$, the total number of individuals will be over $N$. Thus, all individuals of rank smaller than $i^*$ survive to the next generation, and for individuals of rank $i^*$, we take the individuals with the highest crowding distance, breaking ties randomly, so that in total exactly $N$ individuals are kept. In practice, the algorithm is run until some stopping criterion is met. In our mathematical analysis, we are interested in how long it takes until the full Pareto front is covered by the population if the algorithm is not stopped earlier. For that reason, we do not specify a termination criterion. 

To create the offspring, the algorithm selects $N/2$ pairs of individuals from the parent population (possibly with repetition). For each pair, with probability $0.9$, we generate two intermediate offspring via a $2$-offspring uniform crossover (that is, for each position independently, with probability $0.5$, the first child inherits the bit from the first parent, and otherwise from the second parent; the bits from the two parents that are not inherited by the first child make up the second child). Bit-wise mutation is then performed on these two intermediate offspring (that is, each bit is flipped independently with probability~$\frac 1n$). With the remaining $0.1$ probability, this mutation is performed directly on the two parents. 

Different methods can be employed to select the parents that are used to create the offspring population (that is, the aforementioned $N/2$ pairs). 
i)~Fair selection: Each individual appears exactly once in a pair, apart from this, the pairing is random. 
ii)~Uniform selection: Each pair consists of two random individuals. 
iii)~$N$ independent binary tournaments: for $N$ times, uniformly at random sample $2$ different parents and conduct a binary tournament between the two, i.e., select the one with the lower rank, breaking ties by selecting the one with the larger crowding distance, breaking remaining ties randomly; form $N/2$ pairs from the winners randomly. 
iv)~Two-permutation tournament scheme: Generate two random permutations $\pi_{1}$ and $\pi_{2}$ of $P_t$ and conduct a binary tournament between $\pi_{j}(2i-1)$ and $\pi_{j}(2i)$ for all $i\in[1..N/2]$ and $j\in\{1,2\}$; form a pair from the two winners in each interval of length~$4$ in a permutation.

\subsection{The ($\mu+1$) Genetic Algorithm}
The \mpoga maintains a population of $\mu$ individuals which are randomly initialized at the beginning of a run. In each generation, a new individual is created. With a constant probability~$p_c$, it is created by selecting two parents from the
population uniformly at random, crossing them over, and then applying mutation to the resulting offspring. With probability $1-p_c$, a single individual is selected and only mutation is applied. At the end of the generation, the worst individual is removed from the population, with ties broken randomly.

Similarly to our analysis of the NSGA-II, in our analysis of the ($\mu+1$) GA, we consider applying uniform crossover with probability $p_c = 0.9$. Here crossover only produces one child where each bit has $50\%$ chance coming from the first parent and $50\%$ chance coming from the second. Bit-wise mutation of rate $\frac{1}{n}$ is employed.

\subsection{Benchmark Problems}

For $x \in \{0,1\}^n$, let $|x|_0$ and $|x|_1$ denote the number of $0$-bits and $1$-bits in $x$, respectively.  Let $k=[2 .. n/4]$. The function $\jumpnk=f:\{0,1\}^n\rightarrow\mathbb{R}$ was proposed in~\cite{DrosteJW02} and serves as the prime example to study how randomized search heuristics cope with local optima~\cite{JansenW02,JagerskupperS07,DoerrLMN17,HasenohrlS18,LissovoiOW19,CorusOY20,BenbakiBD21,FajardoS21foga,AntipovBD22,Doerr22,RajabiW22,Witt23}. It is defined by
\[f(x) = \begin{cases}
    k+|x|_1, & \text{if }|x|_1 \leq n-k\text{ or } x=1^n,\\
    n-|x|_1, & \text{else},
    \end{cases}\]
for all $x \in \{0,1\}^n$. The aim is to maximize $f$. It has a valley of low fitness around its optimum, which can be crossed only by flipping the $k$ correct bits if no solutions of lower fitness are accepted.

The function $\onejumpzerojump = (f_1, f_2):\{0,1\}^n\rightarrow\mathbb{R}^2$, proposed in~\cite{DoerrZ21aaai}, is defined by
\[f_1(x) = \begin{cases}
    k+|x|_1, & \text{if }|x|_1 \leq n-k\text{ or } x=1^n,\\
    n-|x|_1, & \text{else};
    \end{cases}\]
\[f_2(x) = \begin{cases}
    k+|x|_0, & \text{if }|x|_0 \leq n-k\text{ or } x=0^n,\\
    n-|x|_0, & \text{else};
    \end{cases}\]
for all $x \in \{0,1\}^n$; here $|x|_0$ is the number of $0$-bits in $x$. The aim is to maximize both $f_1$ and $f_2$, two multimodal objectives. The first objective is the classical $\jumpnk$ function. The second objective is isomorphic to the first, with the roles of zeroes and ones exchanged.

According to Theorem~$2$ in~\cite{DoerrZ21aaai}, the Pareto set of the benchmark is $S^*=\{x \in \{0,1\}^n \mid|x|_1 = [k..n-k]\cup\{0, n\}\}$, and the Pareto front $F^*=f(S^*)$ is $\{(a, 2k+n-a)\mid a\in[2k..n]\cup\{k, n+k\}\}$, making the size of the front $n-2k+3$. We define the inner part of the Pareto set by $S_{I}^*=\{x\mid|x|_{1}\in [k..n-k]\}$, and the inner part of the Pareto front by $F_{I}^*=f(S_{I}^*)=\{(a, 2k+n-a)\mid a\in[2k..n]\}$. 

In~\cite{DoerrZ21aaai}, it was proven that the SEMO cannot optimize this benchmark. The GSEMO can and does so in expected time $O(n^{k+1})$. 
In~\cite{DoerrQ22ppsn}, it was shown that when using a population of size $N\geq 4(n-2k+3)$ to optimize this benchmark, the NSGA-II algorithm never loses a Pareto-optimal solution value once found. Then $O(n^k)$ iterations are needed in expectation to find the full Pareto front (all Pareto optimal solution values).

\section{Runtime Analysis of the NSGA-II with Crossover}\label{analysis-nsga}

In this section, we analyze the complexity of the NSGA\nobreakdash-II algorithm with population size $N=c(n-2k+3)$ for some $c>4$. We consider all four different ways of selecting the parents for variation described in Section~\ref{sec:prelim}. For any generation $t$ of a run of the algorithm, we use $P_t$ to denote the parent population and $R_t$ to denote the combined parent and offspring population. For individuals $x$ and $y$, we use $|x|_1$ to denote the number of 1-bits in $x$ and $H(x,y)$ to denote the Hamming distance between $x$ and $y$.

In~\cite{DoerrQ22ppsn}, it was proven that the NSGA-II in an expected time (number of fitness evaluations) of $O(N n^k)$ covers the entire Pareto front of the \ojzj benchmark. This time mostly stems from the waiting times to generate the extremal points on the front, i.e., the all-zeroes string $0^n$ and the all-ones string $1^n$. Since individuals with lower fitness values are not kept in the population, these two extremal points can only be created from individuals with $i\in [k .. n-k]$ 1-bits. Crossing the fitness valley of length $k$ happens for an individual with probability $\Theta(n^{-k})$, resulting in the $O(Nn^k)$ runtime.

To profit from crossover, we exploit the fact that there can be diverse individuals at the outermost points of $F^*_I$. For example if in a $P_t$ there are two individuals $x$ and $y$ such that $|x|_1=|y|_1=n-k$ and $x$ and $y$ are different, i.e., $H(x, y)=2d$ for $d\geq 1$, crossing them over creates an individual $z$ such that $|z|_1=n-k+d$ with probability $\Theta(\frac{1}{2^{2d}})$. Then mutating $z$ gives $1^n$ with probability $\Theta(\frac{1}{n^{k-d}})$. Therefore, the probability that $1^n$ is generated in iteration $t$ is $\Omega(\frac{1}{2^{2d}n^{k-d}})$, making the waiting time smaller than $O(n^k)$. Following this idea, our analysis focuses on the runtime needed to create, maintain, and in the end, take advantage of such diversity at the outermost points of $F^*_I$. 

\begin{lemma}\label{lem:select}
Consider an iteration $t$ of the NSGA\nobreakdash-II algorithm optimizing the $\onejumpzerojump$ benchmark with population size $N = c(n-2k+3)$ for some $c > 4$. Suppose $x\in P_t$ belongs to rank $1$. Then with the fair selection method, the uniform selection method, $N$ independent binary tournaments, or the two-permutation tournament scheme, the probability that $x$ is selected to mutate in iterations $t$ is $\Theta(1)$. 
\end{lemma}

\begin{proof}
The probability that any individual being selected with fair selection is $1$. With uniform selection, the probability that any individual is selected is at least $1-(1-\frac{1}{N})^N\geq 1-\frac{1}{e}$. Now consider using tournaments. Let $F_1$ be all the rank-1 individuals in the current parent population. Then by Corollary~6 of \cite{DoerrZ21aaai}, $|f(F_1)| \leq n-2k+3$. Let $F_1^*$ denote the number of individuals in $F_1$ with positive crowding distances. Then by Lemma~1 of \cite{DoerrQ22ppsn}, $|F_1^*| \leq 4(n-2k+3)$. Suppose $x$ participates in a binary tournament. If the opponent of $x$ is not in $F_1^*$, which happens with probability at least $\frac{(c-4)(n-2k+3)}{c(n-2k+3)}$, then $x$ has at least $\frac{1}{2}$ chance of winning the tournament. With $N$ binary tournaments, the probability that $x$ participates in a particular binary tournament is $\frac{2}{N}$. So the probability that $x$ wins at least one of the $N$ independent tournaments is $1-(1-\frac{2}{N}\frac{1}{2}\frac{c-4}{c})^N \geq 1-e^{-\frac{c-4}{c}}$. With the two independent permutations, $x$ participates in two tournaments, so the probability that $x$ wins at least once is at least $2\frac{1}{2}\frac{c-4}{c}-(\frac{1}{2}\frac{c-4}{c})^2 = \frac{c-4}{c}-\frac{(c-4)^2}{4c^2}$. 
\end{proof}



\begin{corollary}\label{cor:pair_select}
Consider an iteration $t$ of the NSGA\nobreakdash-II algorithm optimizing the $\onejumpzerojump$ benchmark with population size $N = c(n-2k+3)$ for some $c > 4$. With any selection method, the probability that two rank-1 individuals $x, y\in P_t$ are both selected to mutate in iteration $t$ is $\Theta(1)$. 
\end{corollary}

\begin{proof}
With fair selection, the probability that $x$ and $y$ are both selected is $1$. With uniform selection, the probability that $x$ is selected is at least $1-\frac{1}{e}$. Denote the number of times that $x$ is selected by $X$, then $\Pr[\text{x selected}\land\text{y selected}]=\sum_{i=1}^{N-1}\Pr[X=i]\Pr[\text{y selected}\vert X=i]\geq\Pr[X=1]\Pr[\text{y selected}\vert X=1]\geq N\frac{1}{N}(1-\frac{1}{N})^{N-1}(1-(1-\frac{1}{N})^{N-1})=\Theta(1)$. With $N$ independent binary tournaments, given that $x$ is selected, if the opponent of $y$ in the binary tournament is not in $F_1^*$ and not $x$, then $y$ has at least $\frac{1}{2}$ chance of winning the tournament. So $\Pr[\text{x selected}\land\text{y selected}]\geq (1-e^{-\frac{2(c-4)}{c}})(1-\frac{2}{N}\frac{1}{2}\Theta(\frac{c-4}{c}))^N=\Theta(1)$. Similarly, under the two permuatation scheme, $\Pr[\text{x selected}\land\text{y selected}] \geq (\frac{c-4}{c}-\frac{(c-4)^2}{4c^2})(\Theta(\frac{c-4}{c})-\Theta((\frac{1}{2}\frac{c-4}{c})^2))=\Theta(1)$.
\end{proof}


\begin{lemma}\label{lem:survives} Consider an iteration $t$ of the NSGA\nobreakdash-II algorithm optimizing the $\onejumpzerojump$ benchmark with population size $N = c(n-2k+3)$ for some $c > 4$ where $F^*\nsubseteq f(R_t)$. Suppose $x\in R_t$ is an individual of rank $1$. Then with probability greater than $\frac{c-4}{2c}$, we have $x\in P_{t+1}$.
\end{lemma}

\begin{proof}
Let $F_1$ be all the rank-1 individuals in $R_t$. Since $x\in F_1$, if $|F_1|\leq N$, $x$ will survive to generation $t+1$ with probability 1. So consider the case where rank 1 is the critical rank, i.e., the rank where some but not all individuals survive. Since the algorithm has not discovered the entire Pareto front, we have $|f(F_1)| < n-2k+3$. By Lemma~1 of \cite{DoerrQ22ppsn}, there are at most $4|f(F_1)| < N$ individuals in $F_1$ with positive crowding distances. Therefore, if $x$ has a positive crowding distance, it will survive to generation $t+1$ with probability 1. So consider the case where the crowding distance of $x$ is zero. Moreover, there are at least $N-4|f(F_1)| > (c-4)(n-2k+3)$ rank\nobreakdash-1 individuals with zero crowding distance that survive to generation $t+1$. Also, there are less than $2N=2c(n-2k+3)$ individuals in $R_t$ with rank 1 and zero crowding distance. Since all individuals of the critical rank and zero crowding distance are selected for survival uniformly at random, $x$ survives to generation $t+1$ with probability greater than $\frac{(c-4)(n-2k+3)}{2c(n-2k+3)}=\frac{c-4}{2c}$.
\end{proof}

\begin{corollary}
Consider an iteration $t$ of the NSGA\nobreakdash-II algorithm optimizing the $\onejumpzerojump$ benchmark with population size $N = c(n-2k+3)$ for some $c > 4$ where $F^*\nsubseteq f(R_t)$. Suppose $x$ and $y$ are two rank-1 individuals in $R_t$. Then with probability at least $(\frac{c-4}{2c})^2$, we have $x, y\in P_{t+1}$.
\end{corollary}

\begin{proof}\label{cor:pair_survive}
Similarly to the proof of Lemma~\ref{lem:survives}, if both $x$ and $y$ have positive crowding distances, the probability that they both survive is $1$ since there are at most $N$ individuals in $F_1$ with positive crowding distances. Suppose, without loss of generality, that $x$ has a positive crowding distance and $y$ has zero crowding distance. The probability that both of them survive is $\frac{c-4}{2c}$ since the probability of $x$ surviving is $1$ and that of $y$ is $\frac{c-4}{2c}$ according to Lemma~\ref{lem:survives}, and the two events are independent. Suppose both $x$ and $y$ has zero crowding distance, then one of them survives with probability greater than $\frac{c-4}{2c}$. Given that one has survived, there are now $N-4|f(F_1)|-1 \geq (c-4)(n-2k+3)$ spots left for rank-1 individuals with zero crowding distance to survive. So the probability that the other one also survives is still at least $\frac{c-4}{2c}$. So the probability that both of them survive is at least $(\frac{c-4}{2c})^2$.
\end{proof}


Now, we give a lemma on the runtime needed to create and maintain diversity at the outermost points of $F_I^*$. The idea is that in an iteration $t$, through mutation, an individual $x$ with $n-k$ 1-bits can generate an individual $y$ with the same number of 1-bits but different from $x$ (e.g., $H(x, y)=2$), creating diversity among individuals with $n-k$ 1-bits. Then by Lemma~\ref{lem:pair_survive}, $x$ and $y$ survive to iteration $t+1$ with a constant probability, where further diversity can be created in the same way. Accumulating diversity in $k$ iterations, we will have two individuals both with $n-k$ 1-bits and having a Hamming distance of $2k$. 
\begin{lemma}\label{lem:diversifies}
Consider an iteration $t$ of the NSGA\nobreakdash-II algorithm optimizing the $\onejumpzerojump$ benchmark for $k=o(\sqrt{n})$ with population size $N = c(n-2k+3)$ for some $c > 4$. With any of the four parent selection methods, applying uniform crossover with probability $0.9$ and bit-wise mutation, if there is $x\in P_t$ such that $|x|_1 = n-k$, then in another $O(\frac{(Kn)^k}{(k-1)!})$ iterations, for $K=(\frac{2c}{c-4})^2$, in expectation, the parent population will contain $x$ and $y$ such that $|x|_1=|y|_1=n-k$ and $H(x,y)=2k$.
\end{lemma}
\begin{proof}
By Lemma~\ref{lem:select}, the probability that $x$ is selected as a parent is $\Theta(1)$. The probability that no crossover happens with $x$ is $0.1$. The probability that $x$ generates some $x_1$ by bit-wise mutation such that $|x_1|_1 = n-k$ and $H(x,y)=2$ is $\frac{k}{n}\frac{n-k}{n}(1-\frac{1}{n})^{n-2}\geq \frac{1}{e}\frac{k}{n}\frac{n-k}{n}$. By Corollary~\ref{cor:pair_survive}, the probability that both $x$ and $x_1$ survive to generation $t+1$ is at least $(\frac{c-4}{2c})^2$. Therefore, given that $x\in P_t$, the probability that both $x$ and $x_1$ are in $P_{t+1}$ is $\Omega((\frac{c-4}{2c})^2\frac{k}{n}\frac{n-k}{n})$. Similarly, given that $x_1$ is in $P_{t+1}$, the probability that $x_1$ generates $x_2$ such that $|x_2|_1=n-k$ and $H(x, x_2)=4$, and both $x$ and $x_2$ survive to generation $t+2$ is $\Omega((\frac{c-4}{2c})^2\frac{k-1}{n}\frac{n-k-1}{n})$. Continuing this way, the probability that $x$ and $x_k$ exist in $P_{t+k}$ where $|x|_1=|x_k|_1=n-k$ and $H(x, x_k)=2k$ is $\Omega((\frac{c-4}{2c})^{2k}\frac{k!}{n^k}\frac{(n-k)\dots(n-2k+1)}{n^k})=\Omega(\frac{(\frac{c-4}{2c})^{2k}k!}{n^k})$ for $k=o(\sqrt{n})$. Therefore, in expectation, if there is $x\in P_t$ such that $|x|_1=n-k$, in $O(k\frac{(Kn)^k}{k!})=O(\frac{(Kn)^k}{(k-1)!})$ iterations in expectation, for $K=(\frac{2c}{c-4})^2$, the parent population will contain $x$ and $x_k$ such that $|x|_1=|x_k|_1=n-k$ and $H(x, x_k)=2k$. 
\end{proof}

Finally, in the following Theorem we analyze the runtime needed to create the extremal points of the Pareto front, given that we have the desired diversity on the inner part, and combine everything together to obtain the final runtime.

\begin{theorem}\label{thm:runtime}
Consider an iteration $t$ of the \NSGA algorithm optimizing the $\onejumpzerojump$ benchmark for $k=o(\sqrt{n})$ with population size $N = c(n-2k+3)$ for some $c > 4$. Suppose $1^n \notin P_t$. With any of the four parent selection methods, applying uniform crossover with probability $0.9$ and bit-wise mutation, if there is $x\in P_t$ such that $|x|_1=n-k$, then in expectation the algorithm needs another $O(\frac{N^2(Cn)^{k}}{(k-1)!})$ fitness evaluations, for $C=(\frac{4c}{c-4})^2$, to find $1^n$.
\end{theorem}

\begin{proof}
By Lemma~\ref{lem:diversifies}, since there is $x\in P_t$ such that $|x|_1=n-k$, in $O(\frac{(Kn)^k}{(k-1)!})$ iterations, for $K=(\frac{2c}{c-4})^2$, the parent population will contain $x$ and $x_k$ such that $|x|_1=|x_k|_1=n-k$ and $H(x, x_k)=2k$. We call this phase, i.e., producing a population containing such $x$ and $x_k$ from a population containing $x$, the diversification phase.

Denote the generation that $x$ and $x_k$ first appear in the parent population by $t'$. By Corollary~\ref{cor:pair_select}, the probability that they are both selected in generation $t'$ is $\Theta(1)$. The probability that they are then paired up is $\frac{1}{N}$, and the probability that crossover happens on this pair is $0.9$. The probability that one of the intermediate children resulted from crossing $x$ and $x_k$ is the all-ones string is $(\frac{1}{2})^{2k}$. The probability that after mutation the intermediate child remains the all-ones string is $(1-\frac{1}{n})^n = \Theta(\frac{1}{e})=\Theta(1)$. Therefore, the probability that $1^n$ is generated in generation $t'$ is $\Omega(\frac{1}{N}(\frac{1}{2})^{2k})$. 

By Lemma~1 of \cite{DoerrQ22ppsn}, once there is an $x\in P_t$ such that $|x|_1=n-k$, all the future populations will contain an individual with $n-k$ 1-bits. So if $1^n$ is not found in generation $t'$, we can repeat the argument from the beginning of the diversification phase. The expected number of such trials is $(\Omega(\frac{1}{N}(\frac{1}{2})^{2k}))^{-1} = O(N2^{2k})$, and the expected length of one trial, i.e., one run of the diversification phase, is $O(\frac{(Kn)^k}{(k-1)!})$ for $K=(\frac{2c}{c-4})^2$ by Lemma~\ref{lem:diversifies}. Adding all up, once there is an $x\in P_t$ such that $|x|_1=n-k$, in expectation, $O(\frac{(Kn)^k}{(k-1)!}N2^{2k})=O(\frac{N(Cn)^{k}}{(k-1)!})$ iterations, for $C=(\frac{4c}{c-4})^2$, iterations are needed to find $1^n$, corresponding to $O(\frac{N^2(Cn)^{k}}{(k-1)!})$ fitness evaluations.
\end{proof}

Note that the case for finding $0^n$ is symmetrical. Therefore, $O(\frac{N(Cn)^{k}}{(k-1)!})$ iterations, for $C=(\frac{4c}{c-4})^2$, are needed to find both $1^n$ and $0^n$. To cover any other point on the front, since there is a $0.1$ chance that crossover is not applied, the waiting times shown in \cite{DoerrQ22ppsn} do not increase asymptotically. So it does not change from \cite{DoerrQ22ppsn} that to cover the other points a runtime of a lower order is enough. In conclusion, in total, with uniform crossover happening at a constant probability, the number of iterations needed is $O(\frac{N(Cn)^{k}}{(k-1)!})$.

Our analysis has revealed that the advantage of using crossover comes from the fact that now the algorithm does not need to flip all $k$ bits at once to cross the fitness valley. Still, $k$ bits need to be flipped for an individual $x$ with $n-k$ bits of $1$ to create another individual with $n-k$ bits of $1$ that is different enough from $x$ so that the two can produce $1^n$ by only crossover. This explains the $n^k$ term in the runtime proven. However, now these $k$ bits can be flipped one at a time in $k$ iterations, and this happens with probability larger than flipping them all at once since at each iteration the algorithm only needs to flip one bit from multiple available choices. This explains the $\frac{1}{(k-1)!}$ term in the proven runtime. We note that an disadvantage is that flipping one bit at a time means the diversity needed is created in $k$ iterations and in all these $k$ iterations, the diversity that has been created so far needs to be maintained, which is where the $C^k$ term in the runtime comes from. Another disadvantage is that in the end, creating the extremal point through only crossover requires pairing up two particular individuals, for which to happen $N$ iterations are needed in expectation. Therefore, crossover brings real speed-up when $(k-1)!$ outweighs $NC^k$, i.e., when $k\geq c\log(N)/\log\log(N)$ where $c$ is a suitably large constant. Our experiments, however, show that crossover is profitable already from $k=2$ on.

\section{Runtime Analysis of the $(\mu + 1)$ GA}

In~\cite{DangFKKLOSS18}, it was shown that the $(\mu + 1)$ GA with population size $\mu=O(n)$ optimizes the \jumpnk function, $k \ge 3$, in expected time $O(\frac{n^k}{\min\{\mu, n / \log n\}})$. Hence for $\mu = \Theta(n / \log n)$, a speed-up from crossover of $\Omega(n / \log n)$ was shown. The proof of this result builds on a complex analysis of the population dynamics. With additional arguments, the requirements on $\mu$ were relaxed in~\cite{DoerrEJK23arxiv}, but still a logarithmic population size was needed to obtain a speed-up of $\Omega(n)$. In this section, we provide a simpler analysis using the insights from the previous section and prove that crossover is helpful even when $\mu$ is $\Theta(1)$.

Lemma~1 in~\cite{DangFKKLOSS18} shows that the expected time that the algorithm takes so that the entire population reaches the local optimum, i.e., all individuals have $n-k$ bits of $1$, is $O(n\sqrt{k}(\mu \log \mu + \log n))$. We reuse this result, but  now show that after the entire population has reached the local optimum, crossover  decreases the waiting time to find the all-ones string with the $(\mu + 1)$ GA in a similar fashion as for the \NSGA. We note that crossover does not interfere with mutation, so as in~\cite{DoerrQ22ppsn} we have an $\Omega(n^{-k})$ chance to generate the extremal points via mutation and thus have an $O(n^k)$ iterations runtime guarantee also when crossover is used with rate $0.9$.

\begin{lemma}\label{lem:mu_diversifies}
Consider the $(\mu + 1)$ GA optimizing the \jumpnk function with $k = o(\sqrt n)$, applying uniform crossover with probability $0.9$ and bit-wise mutation. Suppose that in some iteration~$t$ the entire population is at the local optimum. Then in another $O((10e\frac{\mu+1}{\mu-1})^k\mu^{k-1}k\frac{n^k}{k!})$ iterations in expectation, the population contains $x$ and $y$ such that $|x|_1=|y|_1=n-k$ and $H(x, y)=2k$.
\end{lemma}

\begin{proof}
The probability that no crossover happens in iteration $t$ is $0.1$. Denote the parent individual chosen in iteration $t$ by~$x$. Since the entire population is at the local optimum, we have $|x|_1=n-k$. Then with probability $\frac{k}{n}\frac{n-k}{n}(1-\frac{1}{n})^{n-2}\geq \frac{1}{e}\frac{k}{n}\frac{n-k}{n}$, the result of applying bit-wise mutation to $x$ is some $x_1$ such that $|x_1|_1=n-k$ and $H(x,y)=2$. Since all the individuals in $P_t$ and $x_1$ have the same fitness, the probability that neither $x$ nor $x_1$ is removed from the population at the end of iteration $t$ is $\frac{\mu-1}{\mu+1}$. Therefore, the probability that $x, x_1\in P_{t+1}$ is at least $\frac{0.1}{e}\frac{k}{n}\frac{n-k}{n}\frac{\mu-1}{\mu+1}$. Now consider iteration $t+1$. The probability that no crossover happens in iteration $t+1$ is $0.1$, and the probability that $x_1$ is chosen for mutation is $\frac{1}{\mu}$. With probability $\frac{k-1}{n}\frac{n-k-1}{n}(1-\frac{1}{n})^{n-2}\geq \frac{1}{e}\frac{k-1}{n}\frac{n-k-1}{n}$, the child of $x_1$ is some $x_2$ such that $|x_2|_1=n-k$ and $H(x, x_2)=4$. Again the probability that neither $x$ nor $x_2$ is removed from the population at the end of iteration $t+1$ is $\frac{\mu-1}{\mu+1}$. Therefore, the probability that $x, x_2\in P_{t+2}$ is at least $\frac{0.1}{e\mu}\frac{k-1}{n}\frac{n-k-1}{n}\frac{\mu-1}{\mu+1}$. Continuing this way, the probability that $x, x_k \in P_{t+k}$ such that $|x|_1=|x|_k=n-k$ and $H(x, x_k)=2k$ is at least $(\frac{1}{10e}\frac{\mu-1}{\mu+1})^k(\frac{1}{\mu})^{k-1}\frac{k!}{n^k}\frac{(n-k)\dots(n-2k+1)}{n^k}=\Omega((\frac{1}{10e}\frac{\mu-1}{\mu+1})^k(\frac{1}{\mu})^{k-1}\frac{k!}{n^k})$ for $k=o(\sqrt{n})$.

Once the algorithm enters the stage where the entire population is at the local optimum, the entire population will remain there until the all-ones string is found, since any newly-generated individuals with fitness levels lower than the local optimum will not be kept. So if the algorithm fails to find such $x$ and $x_k$ in $k$ iterations, the process can start all over again from the beginning. Hence in expectation after the entire population is at the local optimum, in $O((10e\frac{\mu+1}{\mu-1})^k\mu^{k-1}k\frac{n^k}{k!})=O((10e\frac{\mu+1}{\mu-1})^k\mu^{k-1}\frac{n^k}{(k-1)!})$ iterations, the population contains $x$ and $x_k$ such that $|x|_1=|x|_k=n-k$ and $H(x, x_k)=2k$.
\end{proof}

\begin{theorem}\label{thm:runtime_mu}
Consider the $(\mu + 1)$ GA optimizing the \jumpnk function, applying uniform crossover with probability $0.9$ and bit-wise mutation. The number of iterations needed in expectation is $O((40e\mu(\mu+1))^k(\frac{1}{\mu-1})^{k-1}\frac{10e}{9}\frac{n^k}{(k-1)!})+n\sqrt{k}(\mu \log \mu + \log n))$.
\end{theorem}

\begin{proof}
By Lemma $1$ of \cite{DangFKKLOSS18}, in expectation, in $O(n\sqrt{k}(\mu \log \mu + \log n))$ iterations, the entire population is at the local optimum.  Then by Lemma~\ref{lem:mu_diversifies}, in another $O((10e\frac{\mu+1}{\mu-1})^k\mu^{k-1}\frac{n^k}{(k-1)!})$ iterations in expectation, the population contains $x$ and $y$ such that  $|x|_1=|y|_1=n-k$ and $H(x, y)=2k$. We call the phase, i.e., producing a population containing such $x$ and $y$ from a population at the local optimum, the diversification phase.

Denote the generation that $x$ and $y$ first appear in the population by $t$. The probability that crossover happens in iteration $t$ is $0.9$. The probability that $x$ and $y$ are chosen in this iteration is $\frac{1}{\mu(\mu-1)}$. The probability that the result of applying crossover on $x$ and $y$ is the all-ones string is $(\frac{1}{2})^{2k}$, and the probability that no mutation happens on this intermediate child is $(1-\frac{1}{n})^n=\Theta(\frac{1}{e})$. Therefore, the probability that the all-ones string is found in iteration $t$ is $\Theta(\frac{9}{10e}\frac{1}{\mu(\mu-1)}(\frac{1}{4})^k)$.

If the algorithm fails to find the all-ones string in iteration $t$, we can repeat the argument from the beginning of the diversification phase since the population remains at the local optimum. The expected number of such trials is $O(\frac{10e}{9}\mu(\mu-1)4^k)$ and the expected length of one trial is $O((10e\frac{\mu+1}{\mu-1})^k\mu^{k-1}\frac{n^k}{(k-1)!})$ by Lemma~\ref{lem:mu_diversifies}. So after the entire population is at the local optimum, in expectation in another $O((40e\mu(\mu+1))^k(\frac{1}{\mu-1})^{k-1}\frac{10e}{9}\frac{n^k}{(k-1)!})$ iterations, the all-ones string will be found. Adding all together, the algorithm needs in expectation $O((40e\mu(\mu+1))^k(\frac{1}{\mu-1})^{k-1}\frac{10e}{9}\frac{n^k}{(k-1)!})+n\sqrt{k}(\mu \log \mu + \log n))$ iterations to optimize the \jumpnk function.
\end{proof}

The runtime proven in Theorem~\ref{thm:runtime_mu} is minimized when $\mu=\Theta(1)$. Take $\mu=2$ for example. There the speed-up compared with the expected $\Theta(n^k)$ runtime of EAs without crossover is $\Omega(\frac{(k-1)!}{(240e)^k})$, which is a real speed-up for $k=\omega(1)$. In \cite{DangFKKLOSS18}, crossover is only shown to lead to speed up when $\mu$ is large. Specifically, for $k=2$, the best speed-up is observed for $\mu=\Theta(\sqrt{n/\log n})$ and for $k\geq 3$, the best speed-up is observed for $\mu=\Theta(n)$. For $\mu=\Theta(1)$ however, no real speed-up can be inferred from \cite{DangFKKLOSS18}. Our arguments here demonstrate that crossover is helpful for the $(\mu+1)$~GA even for a small population.

\section{Experiments}
To complement our theoretical results, we also experimentally evaluate some runs of the NSGA\nobreakdash-II algorithm on the \ojzj benchmark, and the $(\mu + 1)$ GA on the \jump benchmark.
\subsection{The NSGA-II Optimizing \ojzj}
\subsubsection{Settings}
We implemented the algorithm as described in Section~\ref{sec:prelim} in Python, and experimented with the following settings.
\begin{itemize}
    \item Problem size $n$: $50$ and $100$.
    \item Jump size $k$: $2$.
    \item Population size $N$: $2(n-2k+3)$, and $4(n-2k+3)$. \cite{DoerrQ22ppsn} suggested that, even though their mathematical analysis applies only for $N\geq 4(n-2k+3)$, already for $N=2(n-2k+3)$ the algorithm still succeeds empirically. Therefore, we have also experimented with $N=2(n-2k+3)$ to confirm that our arguments still apply for the smaller population size.
    \item Selection method: fair selection.
    \item Variation operator: with $p_c$ chance uniform crossover is applied followed by bit-wise mutation for $p_c=0$ and $p_c=0.9$.
    \item Number of independent repetitions per setting: $10$.
\end{itemize}
We note that \cite{DoerrQ22ppsn} already presents some data points on the runtime (for $k=3$, $N/(2n-k+3) = 2, 4, 8$, and $n=20, 30, 40$). Therefore, we have chosen the settings to present new data points, and explore how efficient crossover is with different problem and jump sizes.
\subsubsection{Results}
\begin{table}[t]
\centering
\begin{tabular}{|r|r|r|}
\hline
& $p_c=0$ & $p_c=0.9$ \tabularnewline
\hline
$N=2(n-2k+3)$ & $247{,}617$ & $190{,}577$ \\
$N=4(n-2k+3)$ & $416{,}284$ & $147{,}921$ \\
\hline
\end{tabular}
\caption{Average runtime of the NSGA-II with bit-wise mutation on the \ojzj benchmark with $n=50$ and $k=2$.}\label{tab1}
\end{table}
\begin{table}[t]
\centering
\begin{tabular}{|r|r|r|}
\hline
& $p_c=0$ & $p_c=0.9$ \tabularnewline
\hline
$N=2(n-2k+3)$ & $2{,}411{,}383$ & $1{,}954{,}681$\\
$N=4(n-2k+3)$ & $3{,}858{,}084$ & $1{,}322{,}046$\\
\hline
\end{tabular}
\caption{Average runtime of the NSGA-II with bit-wise mutation on the \ojzj benchmark with $n=100$ and $k=2$.}\label{tab2}
\end{table}
Table \ref{tab1} contains the average number of fitness evaluations needed for the NSGA-II algorithm to cover the entire Pareto front of the \onejumpzerojump benchmark for $n=50$ and $k=2$, and Table~\ref{tab2} contains that for $n=100$ and $k=2$. For all of the settings, we have observed a standard deviation that is between $50\%$ to $80\%$ of the mean, and we believe that this is because the runtime is dominated by the waiting time needed to find the extremal points on the front, which roughly speaking could be the maximum of two geometric random variables (one for each extremal point). 
The experimental results support our arguments that crossover significantly speeds-up the \NSGA. 

\cite{DoerrQ22ppsn} has reported runtimes of the NSGA-II algorithm optimizing \onejumpzerojump for $n=20, 30, 40$, for $k=3$, and $N/(n-2k+3)=2, 4, 8$. Heree the version with crossover was faster than the one without by a factor of between $3.9$ and $9$. We note that, for the larger problem sizes $n=50, 100$ and the smaller jump size $k=2$, the effect of crossover is less profound (namely only by a factor of between $1.2$ and $2.9$). This supports the impression given by our mathematical analysis that the speed-up from crossover is more significant when $k$ is large. 

It is also worth noting that the benefits of crossover is much more pronounced when $N=4(n-2k+3)$ than when $N=2(n-2k+3)$. We suspect that this is because when the population becomes larger, there will be more individuals close to the outermost points of $F_I^*$. As a result, it is easier for the algorithm to reach and maintain diversity at the those points, which as our analysis has suggested, is what makes crossover beneficial.

To support our reasoning on how efficient crossover could be for different population sizes, we have also recorded for $n=50$ the diversity among the individuals with $k$ and $n-k$ bits of $1$ throughout the runs of the experiments. Specifically, for every $n^k/50$ iterations, we look at the individuals in the parent population with $k$ and $n-k$ bits of $1$. Among all individuals with $k$ bits of $1$, we calculate the Hamming distances between any two individuals, and record the maximum distance divided by $2$. This number gives us an idea of how much a lucky crossover can decrease the waiting time of finding $0^n$, since as discussed in our analysis, pairing up two individuals with $k$ bits of $1$ whose hamming distance is $2d$ means there is an $\Omega(\frac{1}{2^{2d}n^{k-d}})$ chance of creating the all-zeroes string. We do the same for the individuals with $n-k$ bits of $1$. Note that the greatest this number can be is $k$. Since as shown in~\cite{DoerrQ22ppsn}, the runtime is dominated by the waiting time needed to find the all-ones and the all-zeroes strings after the inner part of the Pareto front has been covered, we are mostly interested in how the diversity develops in that phase. To this end, we discard data points recorded when the inner part of the Pareto front has not been fully covered, and those recorded after one of the extremal points has already been discovered. For one run of the experiment, we average the data points recorded for $k$ and $n-k$ bits of $1$ together in the end. Finally we average the mean values obtained from the $10$ repetitions. For $N=2(n-2k+3)$, we have observed that the diversity measure is $0.76$\textsubscript{$\pm 0.62$}, while for $N=4(n-2k+3)$, it is $0.99$\textsubscript{$\pm 0.45$}. This means indeed there is more diversity on the boundaries with a larger population, explaining why with $N=4(n-2k+3)$ the speed-up from crossover is more than that with $N=2(n-2k+3)$. We note that this does not mean that large population sizes are preferable since the higher cost of one iteration has to be taken into account. Due to the high runtimes, we do not have yet data allowing us to answer this question conclusively, but preliminary experiments for $n=50$ suggest that a population size of $N=8(n-2k+3)$ already gives inferior results.

We note that the speed-up from crossover empirically is more profound than what we have shown theoretically. Though it is normal that mathematical runtime guarantees cannot capture the full strength of an algorithm, as a first attempt to understand where the discrepancy comes from we also recorded for each run how the all-ones string is generated (for the case $n=100$ and $N=2(n-2k+3)$). As a result, among the $10$ runs that we have done, $9$ times crossover has participated in the generation of $1^n$. Out of these $9$ times, there are $5$ times where crossover by itself has created $1^n$. However, for all of those $5$ times, only one of the parents has $n-k=98$ 1-bits, while the other parent has between $95$ to $97$ 1-bits. Among all the runs, we have observed only once that both of the parents have $98$ bits of $1$. This suggests that crossover also profits from diversity between individuals of different objective values, a fact not exploited in our runtime analysis.

\subsection{The ($\mu+1$) GA Optimizing \jump}
\begin{figure}[t]
\centering
\includegraphics[width=0.9\columnwidth]{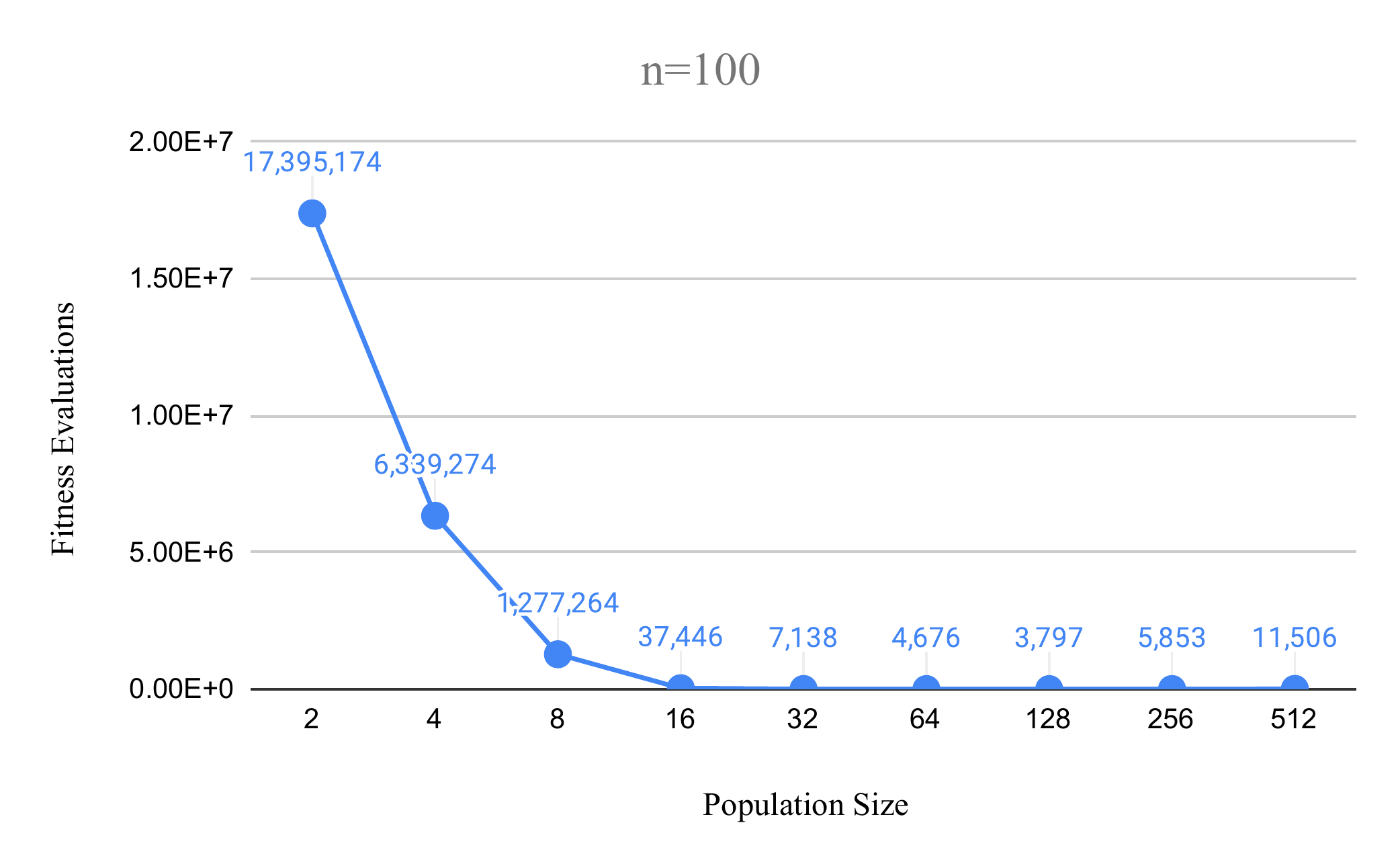}
\caption{Average number of fitness evaluations needed for the ($\mu+1$) GA to optimize \jumpnk for $n=100$ and $k=4$ using uniform crossover with probability $0.9$ and bit-wise mutation.}
\label{fig1}
\end{figure}
\begin{figure}[t]
\centering
\includegraphics[width=0.9\columnwidth]{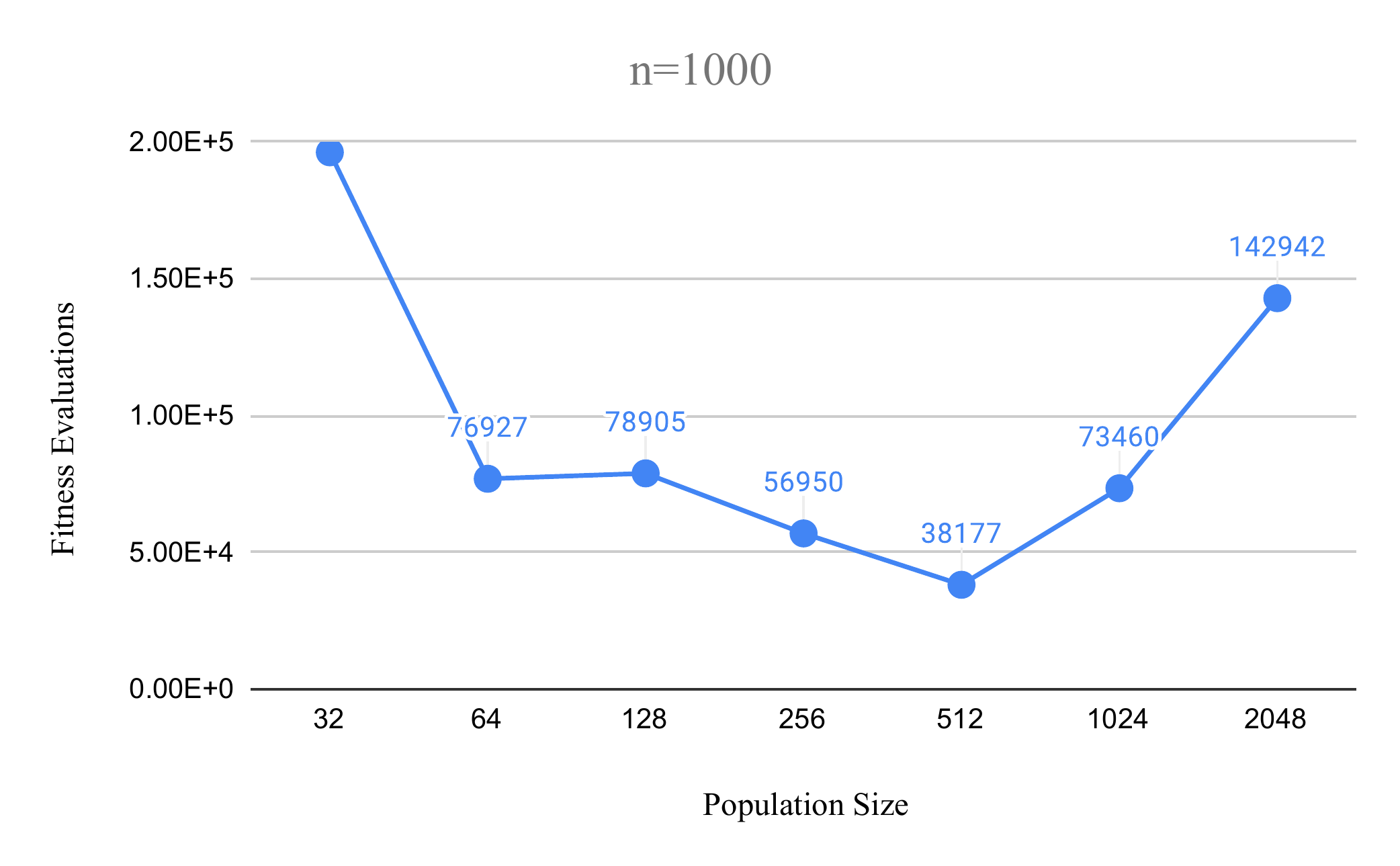}
\caption{Average number of fitness evaluations needed for the ($\mu+1$) GA to optimize \jumpnk for $n=1000$ and $k=4$ using uniform crossover with probability $0.9$ and bit-wise mutation.}
\label{fig2}
\end{figure}
\subsubsection{Settings}
We implemented the algorithm as described in Section~\ref{sec:prelim} in Python, and experimented with the following settings.
\begin{itemize}
    \item Problem size $n$: $100$ and $1000$.
    \item Jump size $k$: $4$.
    \item Population size $\mu$: $2^i$ for $i=[0 .. 9]$ for $n=100$, and $2^i$ for $i=[5 .. 11]$ for $n=1000$. These population sizes are chosen for us to explore the best population size for each problem size.
    \item Variation operator: with $0.9$ chance uniform crossover is applied followed by bit-wise mutation, and otherwise only bit-wise mutation is performed.
    \item Number of independent repetitions per setting: $10$.
\end{itemize}
\subsubsection{Runtime}
Figure \ref{fig1} contains the average number of fitness evaluations needed for the ($\mu+1$) GA to optimize \jumpnk for $n=100$ and Figure \ref{fig2} contains that for $n=1000$. 

The results confirm that crossover is beneficial for the ($\mu+1$) GA optimizing \jump as predicted from the runtime analyses in~\cite{DangFKKLOSS18} and this work. The experimental results suggest that the best speed-up is observed for large populations ($128$ for $n=100$ and $512$ for $n=1000$), different from our analysis. However, already smaller populations give significant speed-ups. For example, when $n=100$ and $k=4$, merely with $\mu=2$, crossover is able to decrease the runtime from over $((1/n)^{k} (1-1/n)^{n-k})^{-1} \ge 10^{8}$, the expected waiting time for a successful mutation, to $1.7\times 10^7$. With $\mu=4$, the runtime is further decreased to $6.3\times 10^6$. These results suggest that the effects exploited in our analysis contribute measurably significantly to the advantages of crossover, in particular, when $\mu=\Theta(1)$ where the analysis in~\cite{DangFKKLOSS18} could not show a speed-up.

\section{Conclusions and Future Works}

In this work, we conducted a mathematical runtime analysis of the \NSGA that shows a speed-up from crossover (on the \ojzj benchmark), already for small population sizes (larger than four times the Pareto front, a number required already for previous runtime analyses of a mutation-only version of the \NSGA). Interestingly, the proven gain from crossover increases with the difficulty parameter $k$ of the \ojzj benchmark. 
Our results are very different from previous runtime analyses of crossover-based algorithms. With no runtime analyses of MOEAs on the \ojzj benchmark or comparable problems existing, the work closest to ours might be the runtime analysis of the \mpoga on the single-objective \jump benchmark in~\cite{DangFKKLOSS18}. There the best speed-ups were obtained from population sizes of order $\sqrt n$, and the speed-ups were at most a factor of $\Theta(n)$, regardless of how large the difficulty parameter $k$ was chosen.
Hence our work has detected a novel way how crossover-based algorithms can leave local optima. As a side result, we show that our arguments can  be employed for the \mpoga, showing significant speed-ups there as well, again from small (constant) population sizes on and increasing significantly with the difficulty parameter~$k$.

Our experimental results, similar to the ones in~\cite{DangFKKLOSS18}, confirm the proven advantages of crossover, but also show that crossover is much more powerful than what the mathematical proofs could distill. For the \NSGA, for example, we observe considerable speed-ups already for the smallest possible value $k=2$ of the difficulty parameter. Trying to explain these via proven results appears as an interesting problem for future research. 

As a second direction for further research, we note that we did not prove any lower bounds, so we have no estimate on how far our runtime guarantees are from the truth. So far lower bounds have only been proven for the mutation-based \NSGA~\cite{DoerrQ23LB}. Understanding the population dynamics of a crossover-based algorithm well enough to admit reasonable lower bounds surely needs significantly stronger arguments since now it is not sufficient to understand how many individuals with a certain objective value exist in the population, but also their genotype is important.

\section{Acknowledgments}

This work was supported by a public grant as part of the
Investissements d'avenir project, reference ANR-11-LABX-0056-LMH,
LabEx LMH.

\bibliographystyle{alphaurl}
\bibliography{ich_master,alles_ea_master}

}
\end{document}